\documentclass[twoside,leqno,twocolumn]{article}

\usepackage{ltexpprt}
\usepackage{amsmath,dsfont,amssymb}
\usepackage{graphicx}

\usepackage[utf8]{inputenc}
\usepackage[english]{babel}

\usepackage{xcolor}
\author{%
Myriam Tami$^1$, Marianne Clausel$^2$, Emilie Devijver$^1$, Adrien Dulac$^1$,  Eric Gaussier$^1$,\\ Stefan Janaqi$^3$ \& Meriam Chebre$^4$\\[1ex]
\textit{ 
$^1$ University Grenoble Alpes, CNRS, Grenoble INP\footnote{Institute of Engineering Univ. Grenoble Alpes}, LIG, F-38000 Grenoble France}\\
\textit{ $^2$ University of Lorraine, Nancy · IEC - Institut Elie Cartan, 54052 Nancy France}\\
\textit{ $^3$ Ecole des Mines d’Alès, 6 avenue de Clavières, F-30319 Ales Cedex France}\\
\textit{ $^4$ TOTAL S.A., 24 cours Michelet - Défense 10, 92069 Paris La Défense Cedex France
}
}

\begin{document}

\title{\Large Uncertain Trees: Dealing with Uncertain Inputs in Regression Trees}

\date{}
\maketitle

\begin{abstract} \small\baselineskip=9pt Tree-based ensemble methods, as Random Forests and Gradient Boosted Trees, have been successfully used for regression in many applications and research studies. Furthermore, these methods have been extended in order to deal with uncertainty in the output variable, using for example a quantile loss in Random Forests \cite{meinshausen2006quantile}. To the best of our knowledge, no extension has been provided yet for dealing with uncertainties in the input variables, even though such uncertainties are common in practical situations.
We propose here such an extension by showing how standard regression trees optimizing a quadratic loss can be adapted and learned while taking into account the uncertainties in the inputs. By doing  so, one no longer assumes that an observation lies into a single region of the regression tree, but rather that it belongs to each region with a certain probability. Experiments conducted on several data sets illustrate the good behavior of the proposed extension.
\end{abstract}

\section{Introduction}

Most measures carried out in the real world, \textit{e.g.}, by sensors embedded in different machines or by analyses of samples, are uncertain if not erroneous in some cases. This uncertainty may be due to the generating process of the samples being measured or from the intrinsic limitations of any measurement process. Considering such measures, that constitute many of the data sets used in data science applications in both industry and academy, as certain is thus in the best case a naive position. Our experiments illustrate this point inasmuch as the method we propose here to handle uncertainty outperforms standard approaches on several, real benchmark data sets.

However, if several methods have been developed to obtain uncertainty estimates from data sets, very few studies have been devoted to designing data science methods that can deal with such uncertainties. We address this problem here in the context of regression trees, a popular machine learning method at the basis of widely used ensemble methods as Random Forests.

In this context, recent studies have focused on providing conditional quantiles, as opposed to conditional means, so as to better represent the output variable and avoid the uncertainty inherent to point estimates. The work on quantile regression forests developed by Meinshausen \cite{meinshausen2006quantile} is a good illustration of this. We take here a different approach and directly model the uncertainty of the input variables in the regression trees we consider. This leads to a regression tree in which observations no longer belong to a single leaf. Instead, each observation has a non-null probability of being assigned to any leaf of the tree. The construction process associated to such trees is slightly more complex than the ones of standard trees (it involves in particular the inversion of a $K \times K$ matrix, where $K$ is the number of leaves of the tree), but the improvements obtained on the prediction fully justify this additional complexity, as shown in Section~\ref{sec:exps} on both benchmark and modified (with an additional noise) data sets.

The idea of including information on the uncertainty of the input data in a regression method is not new and is related to uncertainty propagation and sensitivity analysis. Several authors have indeed proposed to integrate uncertainties on the inputs in different regression methods, as multivariate linear models \cite{reis2005} or neural networks \cite{gal2016}. To the best of our knowledge, our approach is the first one to address this problem in the context of regression trees and ensemble methods based on such trees, as Random Forests.

The remainder of the paper is organized as follows: Section~\ref{sec:model} presents the general model we rely on and its main properties; Section~\ref{sec:algo} then describes the algorithms for constructing the regression tree and the associated prediction rule, while Section~\ref{sec:exps} presents the experiments conducted and the numerical results obtained. Finally, Section~\ref{sec:disc} positions our work wrt previous studies while Section~\ref{sec:concl} concludes the paper.

\section{Regression trees with uncertain input}
\label{sec:model}

Let $Y$ be an output random variable and $\boldsymbol{X} = (X^1,\cdots,X^p)$ a $p$-dimensional input random vector. A classical question is to find some relationship between $\boldsymbol{X}$ and $Y$, estimating the so called link function $f$ involved in the model $Y=f(\boldsymbol{X})+\epsilon$.

Tree-based ensemble methods, as Random Forests or Gradient Boosted Trees, are popular machine learning methods, developed to address the above regression problem \cite{HastieTibshiraniFriedman}. In these methods, a set of regressors are constructed and aggregated in a convenient way. The building blocks are decision regression trees \cite{breiman1984classification}, which are defined from a partition of the space ${\cal X}$ of input variables into $K$ regions $(R_k)_{1\leq k \leq K}$ obtained by dyadic splits minimizing a risk function. A weight $\gamma_k$ is associated to each region leading to a piece-wise predictor, for a new input $\boldsymbol{x} \in \mathcal{X}$, of the form:
\begin{align}
T\left(\boldsymbol{x}; \Theta\right) = \sum_{k=1}^K \gamma_k \mathds{1}_{\lbrace \boldsymbol{x} \in R_k \rbrace}
\label{treeClassical}
\end{align}
where $\Theta = (R_k, \gamma_k)_{1\leq k \leq K}$ is the set of parameters, learned from a training data set, defining the tree. Both categorical and quantitative inputs can in theory be considered. For the sake of simplicity, we however focus in this study on quantitative inputs, thus considering that ${\cal X} = \mathbb{R}^p$.

To deal with uncertainty in the inputs, we introduce an auxiliary latent random vector $\boldsymbol{U}$ representing the true value of the data and consider the general regression function that relates $\boldsymbol{U}|\boldsymbol{X}$ to $Y$ through $Y=f\left(\boldsymbol{U}|\boldsymbol{X}\right)+\epsilon$. The standard regression model is obtained from this general model by considering that $U^j | X^j$ is a Dirac at $X^j=x^j$ (or equivalently is Gaussian with mean $x^j$ and variance $0$). We further assume here that the variables are independent of each other and that the true measure given the observation, on each variable, is Gaussian, leading to the following complete model:
\begin{equation}
\left\{
\begin{array}{l}
Y=f\left(\boldsymbol{U}|\boldsymbol{X}\right)+\epsilon, \\
U^j | X^j=x^j \sim {\cal N} \left(x^j, \sigma_{U_j}^2\right) \text{ for } 1\leq j \leq p.
\end{array}
\right.
\label{model}
\end{equation}
The Gaussian distribution is widespread, can be used to approximate several other distributions and is easy to manipulate, hence its use here. Other distributions can nevertheless be considered for the measurement error, but this is beyond the scope of this study. In the remainder, we will denote $\boldsymbol{\sigma}_{\boldsymbol{U}}^2 = ({\sigma}_{U_1}^2, \ldots, {\sigma}_{U_p}^2)$ the vector of variances of the Gaussian distributions. In practical situations, these variances may be given (for example when the data corresponds to measurements from machines for which the uncertainty is known) or may be directly learned from the data. 

When one is dealing with uncertain inputs, and we want to stress again that this is the general situation, one can no longer assume that observations are hard assigned to regions. Instead, each observation has a probability of being associated to each region, leading to the following prediction rule that directly generalizes Eq.~\eqref{treeClassical}:
\begin{align}
    T_{\text{un}}\left(\boldsymbol{x}; \Theta\right) = \sum_{k=1}^K \gamma_k \mathbb{P}( \boldsymbol{U}\in R_k | \boldsymbol{X} =  \boldsymbol{x}, \boldsymbol{\sigma}_{\boldsymbol{U}}^2).
    \label{treeUncertainty}
\end{align}
Note that the set of parameters $\Theta$ now includes the variances $\boldsymbol{\sigma}_{\boldsymbol{U}}^2$. In addition, because of the independence assumption at the basis of the model retained, one has, for $1\leq k\leq K$:
\begin{equation}
\mathbb{P}(\boldsymbol{U}\in R_k | \boldsymbol{X} =  \boldsymbol{x}, \boldsymbol{\sigma}_{\boldsymbol{U}}^2) = \prod_{j=1}^{p} \mathbb{P}(U^j \in R_{k}^j| X^j = x^j,\sigma_{U_j}^2), \nonumber
\end{equation}
where the interval $R_k^j$ corresponds to the region  $R_k$ projected onto the $j$th variable. We now turn to the procedure for learning the parameters of the model.

\subsection{Estimation procedure}

The estimation procedure of the parameters is based, in this study, on the minimization of the empirical quadratic risk, which is the standard risk considered in regression trees. More precisely, for the learning set defined by:
 \begin{align}
    {\cal L}_n&=\left\{\left(x_i^1,\cdots,x_i^{p} , y_i\right)_{1\leq i \leq n}
    \right\},
\label{learning_set}
\end{align}
with $(\boldsymbol{x}_i, y_i)_{1\leq i \leq n}$ the observed sample, we define the empirical risk:
\begin{align}
    R_{emp}(T_{\text{un}}(.;\hat{\Theta}), {\cal L}_n) &= 
    \frac{1}{n} \sum_{i=1}^n (y_i - T_{\text{un}}(\boldsymbol{x}_i;\hat{\Theta}))^2
\label{empiricalRisk}
\end{align}
where $T_{\text{un}}$ has been introduced in Eq.~\eqref{treeUncertainty}. This criterion has to be minimized on the training set wrt the parameters denoted by $\Theta~=~\left\{(R_k, \gamma_k)_{1\leq k \leq K},\boldsymbol{\sigma}_{\boldsymbol{U}}^2\right\} $: regions of the tree, associated weights, and the variances of  the vector $\boldsymbol{U|X}$. To do so, we introduce  the matrix $P \in M_{n,K} (\mathbb{R})$\footnote{$P$ depends on the number of regions considered. It is thus a dynamic matrix that evolves during the construction of the tree. For clarity sake, we do not explicit this dependence in our notation but the reader has to keep this in mind.} defined, for $1\leq i\leq n$ and for $1\leq k \leq K$, by:
\begin{eqnarray}\label{eq:def-P}
P_{i,k} & = &\mathbb{P}\left(\boldsymbol{U}_i \in R_{k}| \boldsymbol{X}_i = \boldsymbol{x}_i ,\boldsymbol{\sigma}_{\boldsymbol{U}}^2\right) \nonumber \\
 & = & \prod_{j=1}^{p} \mathbb{P}\left(U_i^j \in R_{k}^j| X^j_i = x^j_i,\sigma_{U_j}^2 \right). \nonumber
\end{eqnarray}

\noindent \textbf{Estimating $\boldsymbol{\gamma}$ --} It is easy to see that when fixing the regions $(R_k)_{1\leq k \leq K}$ and the vector of variances $\boldsymbol{\sigma}_{\boldsymbol{U}}^2 $, minimizing Eq.~\eqref{empiricalRisk} with respect to $\boldsymbol{\gamma}$ corresponds to a weighted average of $y_1,\ldots, y_n$, in a way similar to the linear regression model if $P^T P$ is not singular:
\begin{align}
    \label{notre_estimateur_gamma_k}
       \hat{\boldsymbol{\gamma}} =& \underset{\boldsymbol{\gamma} \in \mathbb{R}^K}{\operatorname{argmin}} \left\{ R_{emp}(T_{\text{un}}(.;\hat{\Theta}), {\cal L}_n) \right\} \\
       =& \left(P^T P\right)^{-1}  P^T  \boldsymbol{y},\nonumber  
\end{align}
where $\boldsymbol{y}$ denotes the vector of $n$ univariate outputs $y_1,\ldots, y_n$.
Indeed, by definition:
\begin{align*}
    &R_{emp}(T_{\text{un}}(.;{\Theta}), {\cal L}_n)
    = 
    \frac{1}{n} \sum_{i=1}^n \left(y_i - \sum_{k=1}^K \gamma_k {P}_{i,k}\right)^2.
\end{align*}
Differentiating wrt $\gamma_k$, for all $1\leq k \leq K$, one gets:
\begin{align*}
   &\frac{\partial}{\partial{\gamma}_{k}} R_{emp}\left( T_{\text{un}}(.;{\Theta}), {\cal L}_n\right) = 0 \\
    &\Leftrightarrow
   2 \sum_{i=1}^n \left[    
    P_{i,k} \cdot \left( y_i -  \sum_{k'=1}^K  P_{i,k'}  {\gamma}_{k'} \right)\right] = 0\\
   & \Leftrightarrow
   \sum_{i=1}^n 
        P_{i,k} \cdot   y_i
       =   \sum_{i=1}^n   \sum_{k'=1}^K    P_{i,k}  P_{i,k'}  \boldsymbol{\gamma}_{k'}.
\end{align*}
So that, if ${P}^T P$ is not singular:
\begin{align*}
& P^T   \boldsymbol{y}
       =  {P}^T   P\boldsymbol{\gamma} \, \Leftrightarrow \boldsymbol{\gamma} = \left(P^T P\right)^{-1}  P^T  \boldsymbol{y},
\end{align*}
which is a minimum. In Section \ref{sec:invert}, we derive assumptions under which ${P}^T P$ is indeed not singular. In practice, one can always use the pseudo-inverse of ${P}^T P$, which we will do in our experiments. Lastly, note that $\hat{\boldsymbol{\gamma}}$ depends on the regions through the matrix $P$. 

\noindent \textbf{Estimating $(R_k)_{1 \le k \le K}$ --} The regions are constructed in a way similar to the construction process of standard trees with the difference that here, during the construction process of the regression tree, the number of regions $K$ is not fixed and increases step by step, which implies that the size of the matrix $P$ is also changing during the iterative process. Let us assume that $K$ regions have already been identified, meaning that the current tree has $K$ leaves, each leaf corresponding to a region (\textit{i.e.}, hyper-rectangle). 
As in standard regression trees, each \textit{current}  region $R_k, \, 1 \le k \le K$, can be decomposed into two sub-regions wrt a variable $X^j$ and a splitting point $s_k^j$, for $1\leq j \leq p$:
\begin{equation*}
\left\{
\begin{array}{l}
R^j_{k,L}=\{\boldsymbol{x}\in R_k|\, X^j<s_k^j \} \\
R^j_{k,R}=\{\boldsymbol{x} \in R_k|\,X^j>s_k^j\}.
\end{array}
\right.
\end{equation*}
This decomposition adds a new region  for which the associated elements, $P\in M_{n,K+1}(\mathbb{R})$ and  $\boldsymbol{\gamma} \in \mathbb{R}^{K+1}$ can readily be computed. $P$ has now $K+1$ regions corresponding to the current regions $R_{k'}$ (with $k' \ne k$) and the two new regions $R^j_{k,L}$ and $R^j_{k,R}$. For each current region $R_k$, one is looking for the best split, \textit{i.e.}, the best variable $X^j$ and the best splitting point $s^j_k$ that minimize:
\begin{align*}
\sum_{i=1}^n \left(y_i - \sum_{l=1}^{K+1} \gamma_l {P}_{i,l}\right)^2.
\end{align*}
The sum includes all possible observations as each observation has a non null probability to belong to any region. 
Using Eq. \eqref{notre_estimateur_gamma_k}, the above problem can be reformulated, for each current region $R_k$, as:
\begin{align}
\underset{1\leq j \leq p, s \in \mathcal{S}_k^j}{\operatorname{argmin}} \left\{  \sum_{i=1}^n \left(y_i - \sum_{l=1}^{K+1} \left(\left(P^T P\right)^{-1}  P^T  \boldsymbol{y}\right)_l {P}_{i,l}\right)^2 \right\},
\label{min-split}\end{align}
where $\mathcal{S}_k^j$ denotes the set of splitting points corresponding to the middle points of the $j^{\text{th}}$ coordinates of the observations in $R_k$ sorted according to $X^j$.

The decomposition corresponding to the best split is then used to build the child nodes of $R_k$, which is replaced, in the set of current regions, by its two children. In this process, that is repeated till a stopping criterion is met\footnote{Any standard stopping criterion can be used here.}, the number of regions, the matrix $P$ and the weights $\boldsymbol{\gamma}$ are gradually updated. Section~\ref{sec:algo} provides the algorithm corresponding to this construction.

\noindent \textbf{Estimating $\boldsymbol{\sigma}_{\boldsymbol{U}}^2$ --} Lastly, the vector of variances, $\boldsymbol{\sigma}_{\boldsymbol{U}}^2$,
can either be set according to some high level principles, or can be learned through a grid search on a validation set guaranteeing. The latter is more demanding wrt computational resources, but is likely to lead to better results. However, in our experiments, we use the former strategy, with the aim of showing that our approach is robust in the sense that it yields good results even when $\boldsymbol{\sigma_U}$ is set \textit{a priori}.

\subsection{Final prediction}

For a new observation $\boldsymbol{x}\in \mathbb{R}^p$, one first computes its distribution over all the obtained regions:
$$ \forall k, \, 1\leq k \leq K, \, P^{\boldsymbol{x}}_k = \mathbb{P}\left(\boldsymbol{U} \in {R}_{k}| \boldsymbol{X} = \boldsymbol{x}, {\boldsymbol{\sigma}}_{\boldsymbol{U}}^2 \right) .$$
The prediction is then a direct application of \eqref{treeUncertainty}:
\begin{align}
    T_{\text{un}}\left(\boldsymbol{x}; \hat{\Theta}\right) = \sum_{k=1}^K \hat{\gamma}_k P^{\boldsymbol{x}}_k,
  \label{pred:untree}
\end{align}
where $\boldsymbol{\gamma}$, $(R_k)_{1 \le k \le K}$ and $\boldsymbol{\sigma}_{\boldsymbol{U}}^2$ are estimated as described above.
\subsection{A sufficient condition on the invertibility of $P^T P$}
\label{sec:invert}
 
The formula used for the construction of the tree relies on the inverse on the matrix $P^T P\in M_{K,K}(\mathbb{R})$. Even if numerically the Moore-Penrose pseudo-inverse might be used to approximate this inverse, we provide here a sufficient condition on the invertibility of $P$. Without loss of generality, the regions involved in the definition of the regression tree are of the form $R_k =\prod_{j=1}^p [a_k^j,b_k^j]$ for $1\leq k \leq K$, where $(a_k^j, b_k^j) \in (-\infty, +\infty)^2$. As usual, we denote $q_\alpha$ the $\alpha$ quantile of the standard Gaussian distribution. The general invertibility result is stated in Theorem~\ref{th:gen}.
\begin{theorem}\label{th:gen}
With the model defined in \eqref{model}, if the following assumption is satisfied:
     \[
     \forall j, \, 1\leq j \leq p, \,
     \sigma_{U_j} < \frac{\min\limits_{1\leq k\leq K}(b_k^j-a_k^j)}{2 q_{\frac{1+0.5^{\frac{1}{p}}}{2}}},
     \]
then the matrix $P^TP$ is invertible.
\end{theorem}
Roughly speaking, the matrix $P^T P$ is invertible provided that the standard deviation $\sigma_{U_j}$ is sufficiently small for all $1\leq j \leq p$. The smaller the regions and/or the larger the number $p$ of input variables, the lower the uncertainty on the measurement has to be to ensure the invertibility of the matrix $P^T P$.

To prove Theorem~\ref{th:gen}, we prove that $P$ is of full rank. To do so, we first prove the following result:

\begin{proposition}
\label{prop}
Let us fix $1\leq k \leq K$ and consider $1\leq i \leq n$ such that $a_k^j<x_i^j<b_k^j$. Assume that 
\begin{equation}\label{assSigma}
    \forall 1\leq j \leq p,\, \sigma_{U_j}<\frac{b_k^j-a_k^j}{2q_{\frac{1+0.5^{1/p}}{2}}}.
\end{equation}
Then, $P_{i,k}>0.5$.
\end{proposition}

\begin{proof}
Observe that a sufficient condition to have $P_{i,k}>0.5$ is
\[
\forall j,\, \frac{1}{\sigma_{U_j}\sqrt{2\pi}}\int_{R_k^j}e^{-\frac{(u-x_i^j)^2}{2\sigma_{U_j}^2}}du>0.5^{1/p}.
\]
We now search a sufficient condition to have the inequality just above. To get this inequality the following condition is sufficient: for all $1\leq j\leq p$,
\begin{align}
&\mathbb{P}\left(\frac{U_i^j-x_i^j}{\sigma_{U_j}}<\frac{b_k^j-x_i^j}{\sigma_{U_j}}|X_i=x_i\right) \nonumber\\
&-\mathbb{P}\left(\frac{U_i^j-x_i^j}{\sigma_{U_j}}<\frac{a_k^j-x_i^j}{\sigma_{U_j}}|X_i=x_i\right)>0.5^{1/p}.\label{eq:diff-proba}
\end{align}
Note that this last condition is satisfied if we have, for all $1\leq j \leq p$,
 \begin{equation*}
    \left\{
\begin{array}{l}
   \mathbb{P}\left(\frac{U_i^j-x_i^j}{\sigma_{U_j}}<\frac{b_k^j-x_i^j}{\sigma_{U_j}}|X_i=x_i\right)>\frac{1+0.5^{1/p}}{2},\\
   \mathbb{P}\left(\frac{U_i^j-x_i^j}{\sigma_{U_j}}<\frac{a_k^j-x_i^j}{\sigma_{U_j}}|X_i=x_i\right)<\frac{1-0.5^{1/p}}{2},
\end{array}
\right.
\end{equation*}
which is equivalent to
\[
\frac{a_k^j-x_i^j}{\sigma_{U_j}}<q_{\frac{1-0.5^{1/p}}{2}}\mbox{ and }\frac{b_k^j-x_i^j}{\sigma_{U_j}}>q_{\frac{1+0.5^{1/p}}{2}}.
\]
Note that $q_{\frac{1-0.5^{1/p}}{2}}=-q_{\frac{1+0.5^{1/p}}{2}}<0$. A sufficient condition is then, for all $1\leq j \leq p$,
\[
\sigma_{U_j}<\min_{1\leq k \leq K} \left\{  \frac{\min\left(x_i^j-a_k^j,b_k^j-x_i^j\right)}{q_{\frac{1+0.5^{1/p}}{2}}}\right\}.
\]
Since  $a_k^j<x_i^j<b_k^j$, a condition even more conservative is the following:
\[
\forall j,\, \sigma_{U_j}<\frac{b_k^j-a_k^j}{2q_{\frac{1+0.5^{1/p}}{2}}}.
\]
This concludes the proof.
\hfill $\blacksquare$
\end{proof}

We assume in the following that Assumption \eqref{assSigma} is satisfied. Then, the set $I_k = \{1\leq i \leq n, P_{i,k} >0.5 \}$ is non-empty. Let us consider, for all $1\leq k\leq K$, $i_k$ a representative of this set and let us introduce the matrix $Q \in M_{K,K}(\mathbb{R})$ defined, for all $1\leq l,k \leq K$, by:
\[
Q_{l,k} = P_{i_l,k}.
\]

\begin{proposition}
Assume that Assumption \eqref{assSigma} holds. Then the $K\times K$ matrix $Q$ is invertible.
\end{proposition}
\begin{proof}
We first show that $Q$ is a strictly dominant diagonal matrix, \textit{i.e.}:
\[
\forall k,\,Q_{kk}>\sum_{k'\neq k}Q_{kk'}.
\]
Indeed, by  Proposition \ref{prop}, we know that $P_{i_k,k}>0.5$ and $k$ is the only region where it is true:
\[
\sum_{l\neq k}Q_{l,k}\leq \sum_{l\neq k}P_{i_l,k} = 1 - P_{i_k,k} <0.5<P_{i_k,k}=Q_{k,k}.
\]
According to Hadamard's Lemma, we know that a strictly dominant diagonal matrix is invertible, which concludes the proof.
\hfill $\blacksquare$
\end{proof}
As $Q$ is invertible, $P$ is of full rank, leading to the fact that $P^T P$ is invertible which concludes the proof of Theorem \ref{th:gen}.

\subsection{Extension to Random Forests}

It is straightforward to use the uncertain regression trees introduced above in Random Forests, leading to \textit{Uncertain Random Forests}. Indeed, each tree of the forest is now an uncertain regression tree that can be constructed as outlined before. Assuming a forest of $\tau$ uncertain trees and denoting $\hat{\boldsymbol{\gamma}}^t, \, 1 \le t \le \tau$, the weights estimated for each tree and $P^{\boldsymbol{x},t}_{k_t}$ the probability distribution of a new observation $\boldsymbol{x}\in \mathbb{R}^p$ over the $K_t$ regions ($1 \le k_t \le K_t$) of the $t^{\text{th}}$ tree ($1 \le t \le \tau$) of the forest, the prediction rule of the uncertain random forest takes the form:
\begin{align}
    RF_{\text{un}}\left(\boldsymbol{x}; \hat{\Theta}_{t, \, 1 \le t \le \tau}\right) = \frac{1}{\tau} \sum_{t=1}^{\tau} \sum_{k_t=1}^{K_t} \hat{\gamma}^t_{k_t} P^{\boldsymbol{x},t}_{k_t}.
  \label{pred:unrf}
\end{align}
As one can note, the above prediction rule is a direct extension of Eq. \eqref{pred:untree}.
 
\section{Associated algorithms}
\label{sec:algo}

Algorithm \ref{algo:buid_tree} describes the construction of uncertain regression trees. This construction parallels the one of standard regression trees except that we consider a matrix $P$ encoding the probability distribution of training data points over regions, which is built dynamically (as the regions and the weights) by adding a new column when a given region is split into two sub-regions.

Each time a region $R$, corresponding to a current leaf of the tree being constructed, is considered (through the pop function in Algorithm \ref{algo:buid_tree}), one identifies the best split $(j^\star, s^\star)$ that minimizes the empirical risk in Eq. \eqref{min-split} among all possible splitting points in $\mathcal{S}_j(R)$ of each variable $j$. The set $\mathcal{S}_j(R)$ is defined by $\mathcal{S}_j(R) = \{ (x_{i_{l+1}}^j+x_{i_{l}}^j)/2 \mid  {1\leq l \leq r-1}\}$, with $(\boldsymbol{x}_{i_1}, \ldots, \boldsymbol{x}_{i_r})$ corresponding to the $r$ observations of the learning set $\mathcal{L}_n$ belonging to the region $R$, sorted such that $x_{i_1}^j \leq x_{i_2}^j \leq \cdots \leq x_{i_r}^j$.

The $k$-th column of $P$ corresponding to the current region is finally replaced by the probability distribution of the training data points over its \textit{left sub-region} (denoted $R_L$) and an additional column is added to $P$ for the \textit{right sub-region} (denoted $R_R$). The weights $\hat{\boldsymbol{\gamma}}$ are easily deduced from $P$ at each step through \eqref{notre_estimateur_gamma_k}.

The algorithm finally outputs the set of regions $\mathcal{R}(S)$ corresponding to the leaves of the tree and the weights $\hat{\boldsymbol{\gamma}}$.

\begin{algorithm}
{\textit{Uncertain regression trees}}

    \noindent \textbf{Input:}     ${\cal L}_n$ 
    
    \noindent\textbf{Initialization:} $K=1, \, S \leftarrow (1,\mathcal{X}), \, P = \mathbf{1}_n$
    
    \noindent\textbf{while} stopping criterion not met
    
            $(k, R) = S.\text{pop}()$
            
            \textbf{for} $j = 1 \text{ to } p$
            
            \hspace{\parindent} Construct $\mathcal{S}_j(R)$ 
            
            \hspace{\parindent} \textbf{for} $ s \in \mathcal{S}_j(R)$
            
            \hspace{\parindent} \hspace{\parindent} $R_L = \{\boldsymbol{x} \in R, x_j \leq s\}$
            
            \hspace{\parindent} \hspace{\parindent} $R_R = \{\boldsymbol{x} \in R, x_j \geq s\}$
            
            \hspace{\parindent} \hspace{\parindent} $Q_L = (\mathbb{P}(U_i \in R_L|X_i = x_i))_{1\leq i \leq n}$ 
              
            \hspace{\parindent} \hspace{\parindent} $Q_R = (\mathbb{P}(U_i \in R_R|X_i = x_i))_{1\leq i \leq n}$
            
            \hspace{\parindent} \hspace{\parindent}$P^{[j,s]} \leftarrow$ Merge $\{ P[, 1:k-1], Q_L, $
            
             \hspace{\parindent} \hspace{\parindent}\hspace{\parindent} \hspace{\parindent} \hspace{\parindent} \hspace{\parindent} $P[, k+1:K], Q_R \} $
            
             \hspace{\parindent} \textbf{end for}
            
             \textbf{end for}
            
            $(j^\star, s^\star)=\underset{1\leq j \leq p, s \in \mathcal{S}_j(R)}{\operatorname{argmin}} \text{risk}(P^{[j,s]})$ \#defined in \eqref{min-split}
            
            $P = P^{[j^\star, s^\star]}$
        
            Update $\hat{\boldsymbol{\gamma}}$ acc. to Eq.~\ref{notre_estimateur_gamma_k}
            
            K = K+1
            
            S.append((k,$R_L$), (K, $R_R$))
            
            \noindent{\textbf{end while}}

    \noindent \textbf{Output:} $(\mathcal{R}(S),\hat{\boldsymbol{\gamma}})$
    \label{algo:buid_tree}

\end{algorithm}

Note that in this version the tree is constructed in a depth-first manner, and that the usual stopping criteria for regression trees can be used (as imposing a minimum number of observations in each leaf or a maximal depth for the tree).

Algorithm \ref{algo:predict} provides the pseudo-code for the prediction rule given in  \eqref{pred:untree} for a new covariate $\boldsymbol{x}^\text{pred}$. 

\begin{algorithm}{\textit{Prediction}}

    \noindent \textbf{Input:}  $(R_{k})_{1 \le k \le K}, \, \hat{\boldsymbol{\gamma}}, \, \boldsymbol{x}^{\text{pred}}, \, {\boldsymbol{\sigma}}_U^2 $

    \noindent $\hat{y}^{\text{pred}} = 0 $
    
    \noindent \textbf{for} $k=1$ to $K$  
    
    $\hat{y}^{\text{pred}} +=  \mathbb{P}(\mathbf{U} \in R_k | \mathbf{X} = \boldsymbol{x}^{\text{pred}}, {\boldsymbol{\sigma}}_U^2) \hat{\gamma}_k$

 \noindent \textbf{end for}
            
    \noindent \textbf{Output:} $\hat{y}^{\text{pred}}$ 
    \label{algo:predict}

\end{algorithm}

\section{Experimental validation}
\label{sec:exps}

The goal of our experiments is to assess how uncertain regression trees:
\begin{itemize}
    \item compare to standard regression trees,
    \item behave in \textit{Uncertain} Random Forests.
\end{itemize}
By standard regression trees we mean here regression trees based on the quadratic risk and the prediction rule given in Eq. \eqref{treeClassical}. In addition, we consider a trade-off between standard and uncertain regression trees based on standard trees (and thus avoiding the complex construction process outlined before) but using the prediction rule given in \eqref{pred:untree}. The matrix $P$ and the weights $\boldsymbol{\gamma}$ are thus computed only once, when the standard trees have been built. The rationale for this method is to rely on the strengths of the two approaches: a simple construction tree process and a robust prediction rule. As one can conjecture, this method yields results in between the two approaches.

We make use in our experiments of four benchmark data sets commonly used in regression tasks and described in Section~\ref{sec:data}. We first use these data sets without any modification, to illustrate the fact that \textit{real} data sets contain uncertain inputs. The uncertain regression trees proposed here indeed outperform standard trees and Random Forests on these data sets, as shown in Section~\ref{sec:bench}. We then modify these data sets by adding a uniform noise bounded by a quantity proportional to the empirical variance of the data. This additional perturbation aims at assessing the robustness of the different methods (standard and uncertain trees) in situations where the input data is highly uncertain. Once again, the results obtained show the good behavior of the uncertain trees and Random Forests (Section~\ref{sec:arti}). In all our experiments, the results are evaluated with the Root Mean Squared Errors (RMSE), which is a standard evaluation measure in regressions tasks. To ensure that all the available data is used for both training and testing, we further rely on 5-fold cross-validation and report the mean RMSE and its standard deviation over the 5 folds.

The stopping criterion for the trees, both standard and uncertain, is based on the fact that all leaves should contain at least ten percent of the training data. For Random Forests, three features are randomly selected for constructing a tree, which roughly corresponds to one third of the features on the data sets considered, a standard proportion in Random Forests.

In this study, the vector of variances is fixed according to some high level principle. In particular, when no additional noise is introduced on the data, the variance for a particular variable $U^j|X^j$ is set to the empirical variance of $X^j$ (in other words, we assume that the variance of the \textit{true} values is of the same order as the variance of the observed values). When some noise is added to the data, the variance of $U^j|X^j$ is set to one half of the variance of the observed, noisy data (in this case, the variance of the true values should be lower than the empirical variance of the observed values; we arbitrarily chose one half in this study).

Lastly, our algorithms have been implemented from \texttt{scikit-learn} \cite{Scikit-learn}, using Cython \cite{Behnel}.

\subsection{Data Sets} 
\label{sec:data}
Experiments are conducted on 4 publicly available data sets, popular in the regression tree literature.  The main characteristics of these data sets are summarized in Table~\ref{tab:data set-description}.
As we focus in this paper on quantitative variables, we simply deleted the categorical variables from the data sets.
Several applications are considered, among which environment (concentration in ozone over a day for the data set \textit{Ozone} introduced in \cite{cornillon}), health (data set \textit{Diabete}, introduced in \cite{Scikit-learn} and used in \cite{LAR}), economy (data set \textit{BigMac} about price of sandwiches, available in R package \verb?alr3?  and used in \cite{meinshausen2006quantile}) or biology (data set \textit{Abalone} introduced in \cite{Warwick} and used in \cite{meinshausen2006quantile} among others).

\begin{table}[ht!]
  \centering
      \begin{tabular}{ccc}
        \hline
        Data set & sample size $n$ & number of variables $p$ \\
        \hline
        BigMac & 69&9 \\
        Ozone & 112 & 9 \\
        Diabete & 442 & 10 \\
        Abalone & 500 & 7\\
        \hline
      \end{tabular}
\caption{Characteristics of data sets used in our experiments, ordered by sample size $n$.}
\label{tab:data set-description}
\end{table}
    
\subsection{Results on benchmark data sets}
\label{sec:bench}
As mentioned in the introduction, data sets are by nature uncertain, so we illustrate our method on the four data sets introduced in Section \ref{sec:data}. 
The noise reflects the uncertainty, so we propose to use the empirical standard deviation vector as the input parameter $\boldsymbol{\sigma}_U$. Table~\ref{tab:data set-RMSE - trees} displays the results obtained for each data set in terms of mean and standard deviation of RMSE using 5-fold cross validation.
Standard trees, standard Random Forests (with $\tau$ = 100 trees) as well as uncertain trees and standard trees with uncertain prediction are compared.

\begin{table*}[ht!]
  \centering
      \begin{tabular}{ccccc}
         Data sets & {BigMac}& {Ozone} & {Diabete}   & {Abalone}  \\
        \hline
        \hline
         Standard tree  & 25.09 (12.3) & 17.82 (4.3) & 60.29 (4.3) & 2.70 (0.3) \\
         Standard RF, $\tau = 100$ & 19.78 (11.0) & 15.86 (4.1) & 58.18 (4.3) & 2.65 (0.4) \\
        Standard tree with uncertain prediction  & 21.49 (8.7) & 16.79 (4.1) & 57.05 (3.7)&  2.41 (0.2) \\
        Uncertain tree & 18.74 (8.9) & 15.39 (3.7) & 56.56 (3.3) &  2.33 (0.3)  \\
        \hline
      \end{tabular}
\caption{Average RMSE based on 5-fold cross-validation for the 4 benchmark data sets. Standard deviations are given in parentheses. For each  uncertain tree based method, $\boldsymbol{\sigma}_{U}$ are set to the empirical standard deviations of the observed input variables.}
\label{tab:data set-RMSE - trees}
\end{table*}

For all data sets, the best performances are achieved by uncertain trees.
As expected, the standard Random Forest are performing better than considering only one tree. However, as one can see in Table \ref{tab:data set-RMSE - trees}, uncertain trees yield better results than standard Random Forests. Performances of standard trees with uncertain prediction are better than the ones of standard trees, however not always better than the ones of standard Random Forests.

Finally, one can note that the standard deviation of uncertain trees is smaller than the standard deviation of other methods, meaning that uncertain trees yield more stable results.

\subsection{Results on artificial uncertain data sets}
\label{sec:arti}
To assess the robustness of uncertain trees to and uncertain Random Forests to uncertainty in the input variable, we add artificial noise, which could correspond to some measure error, to the input observations.
By doing so, one can consider that the variability of the data is coming from two sources, on the one hand the variance in the latent variables (related to the variance of $\mathbf{U}$ with the notations of Section \ref{sec:model}) and on the other hand the variance of the uncertainty (related to the variance of $\mathbf{X} | \mathbf{U}$), both leading to the variance of the observations.
Then, we assume here that $\boldsymbol{\sigma}_U$ is smaller than  the standard deviation of the observations $\boldsymbol{\sigma}_X$, which can be estimated through the data set. Basically it means that the main part of the variance is due to the uncertainty.

To construct artificial uncertain data sets satisfying this condition, we consider in this section \textit{BigMac}, \textit{Ozone}, \textit{Diabete} and \textit{Abalone} data sets introduced in Section \ref{sec:data}.
A noise is added to each observation using the following process. For each observation from an input variable $X^j$, $1\leq j \leq p$, we add a noise generated from the product of a Rademacher variable and a uniform variable coming from the interval $[\hat{\sigma}_{X_j}/10, \hat{\sigma}_{X_j}/4]$.

We consider here standard trees, standard Random Forests with $\tau = 500$ trees, standard trees with uncertain prediction, uncertain trees and uncertain Random Forests with $\tau = 15$ trees. The results obtained are displayed in Table \ref{tab:data set-RMSE - RF}.

\begin{table*}[ht!]
  \centering
      \begin{tabular}{ccccc}
        Uncertain data sets  & {BigMac} &{Ozone} &{Diabete} &{Abalone} \\
        \hline
        \hline
        Standard tree & 22.28 (8.7) & 19.37 (4.1) & 60.47 (2.81) & 2.54 (0.17)\\
        Standard tree with uncertain prediction & 21.68 (9.8) & 17.35 (4.9) & 57.92 (3.43) & 2.40 (0.15)\\
      Uncertain tree & 19.28 (13.4) & 16.87 (6.0) & 58.56 (3.45) & 2.34 (0.20)\\
      \hline
         Standard RF, $\tau = 500$ & 19.25 (7.8) & 15.72 (3.0) & 59.55 (4.39) & 2.64 (0.18) \\
        Uncertain RF, $\tau = 15$ &18.06 (9.3)& 15.49 (3.7) & 55.66 (4.31) & 1.98 (0.12) \\
        \hline
      \end{tabular}
\caption{Average RMSE based on 5-fold cross-validation for the 4 modified data sets. Standard deviations are given in parentheses. On each data set, each observation is modified by a noise generated from the product of a Rademacher variable and a uniform variable coming from the interval $[\frac{\hat{\sigma}_{X_j}}{10}, \frac{\hat{\sigma}_{X_j}}{4}]$. For each uncertain tree-based method, $\boldsymbol{\sigma}_U$ is to half of the empirical standard deviations of the observed (modified) input variables.}
\label{tab:data set-RMSE - RF}

\end{table*}

As one can note, uncertain trees outperforms here again standard trees. RMSE scores are a bit higher than in Table \ref{tab:data set-RMSE - trees}, which is not surprising given the noise added to the data sets.

Similarly, uncertain Random Forests outperform all the other methods, including standard Random Forests even though the number of trees is one order of magnitude less. 

\section{Discussion}
\label{sec:disc}

Regression trees have been introduced in the 1980s through the popular CART algorithm \cite{breiman1984classification}, notably allowing one to deal with both categorical and numerical input variables. They constitute the basic building block of state-of-the-art ensemble methods based on the combination of random trees, notably Random Forests introduced by Breiman in~\cite{breiman2001random} to circumvent the instability of Regression trees \cite{gey2006}. Since Random Forests are particularly well suited for Big Data analysis (see \cite{genuer2017}), many applications have been addressed in various domains with Random Forests, for example in ecology or genomics. \cite{verikas2011} provides a review of the use of Random Forests for data mining purposes. An interesting feature of Random Forests is the possibility to quantify the importance of input variables in the model obtained (see~\cite{genuer2010} for more details on that point). Another interesting feature, which was empirically established, is their robustness to noise. They are thus, to a certain extent, robust to uncertain inputs (even though no mechanism was specifically designed for that). This said, explicitly modelling the uncertainty as done in the uncertain regression trees proposed here allows one to outperform the standard version of  Random Forests, as illustrated in our experiments. 

Several adaptations of ensemble methods for quantile regression have been proposed, as quantile Random Forests or quantile boosting trees \cite{fenske2011identifying, kriegler2007boosting, kriegler2010small,meinshausen2006quantile,zheng2012}. These studies however focus on the uncertainty in the output variable (by producing conditional quantiles rather than a conditional mean) and not on the uncertainties in the inputs, as done here. It is of course possible to combine both approaches, which we intend to do in the future.

Lastly, the idea of including information on the uncertainty of the input data in a regression method is not new and is related to uncertainty propagation and sensitivity analysis. Several authors have indeed proposed to integrate uncertainties on the inputs in different regression methods, as multivariate linear models \cite{reis2005} or neural networks \cite{gal2016}. In each case, the methods have been improved, showing the benefits of explicitly modelling uncertainties in the input data. To the best of our knowledge, our approach is the first one to address this problem in the context of regression trees (and ensemble methods based on such trees). Our conclusion on the benefits of this approach parallels the ones of the above-mentioned studies.

\section{Conclusion}
\label{sec:concl}

We have introduced in this study uncertain regression trees with can deal with uncertain inputs. In such trees, observations no longer belong to a single region, but rather have a non-null probability to be assigned to any region of the tree. This extra flexibility leads nevertheless to a construction process that is more complex than the one underlying standard regression trees and that necessitates the inversion of a square matrix (for which we have theoretically provided a sufficient condition). In practice, we rely on the pseudo-inverse of this matrix. The experimental results fully justify the approach we have proposed and show that uncertain regression trees improve the results of standard regression trees and Random Forests on several benchmark data sets. A similar conclusion is drawn on artificial uncertain data sets obtained from the standard ones by introducing additional uncertainty in the form of a uniform noise.

The methodology developed in this study can also be adapted to the case where some input data are categorical. We plan to work on such an adaptation in the future, by considering different types of uncertainties.

Furthermore, as mentioned before, the vector of variances for modelling uncertainties can easily be learned by grid search on validation sets (typically using cross-validation). One can expect by doing so that the results would further improve. We have set this vector in our experiments to values that we believe are reasonable, so as to show that our approach is robust in the sense that it yields good results even when $\boldsymbol{\sigma_U}$ is set \textit{a priori}. We nevertheless plan to run additional experiments to learn this vector. As a grid search can be easily parallelized, this learning should not impact the running time of the algorithm.

Lastly, exploring the extension of our method to Random Forests or boosting trees, as in \cite{freund1999short,Ridgeway2007}, is another promising research direction. We also intend to explore the use of other empirical loss functions, as the quantile loss used in the definition of quantile Random Forests or quantile boosting trees~\cite{fenske2011identifying, kriegler2007boosting, kriegler2010small,meinshausen2006quantile}.

\end{document}